\documentclass{article}

\usepackage[preprint]{neurips_2025}
\usepackage[utf8]{inputenc}
\usepackage[T1]{fontenc}

\usepackage{bm}
\usepackage{amsmath, amsfonts, mathtools}
\usepackage{microtype, xcolor, hyperref, url}
\usepackage{nicefrac, booktabs, setspace}

\usepackage{amsthm}
\newtheorem{theorem}{Theorem}
\newtheorem{definition}{Definition}
\newtheorem{property}{Property}
\theoremstyle{definition}

\usepackage{graphicx, adjustbox, float}
\usepackage{enumitem}

\usepackage{algorithm, mdframed}
\usepackage[noend]{algpseudocode}

\newmdenv[
  backgroundcolor=gray!8,
  roundcorner=4pt,
  skipabove=\topsep,
  skipbelow=\topsep,
  innertopmargin=6pt,
  innerbottommargin=6pt,
  innerleftmargin=6pt,
  innerrightmargin=6pt
]{examplebox}

\title{$K$-MSHC: Unmasking Minimally Sufficient Head Circuits in Large Language Models with Experiments on Syntactic Classification Tasks}

\author{%
  Pratim Chowdhary\\
  Department of Computer Science\\
  Dartmouth College\\
  \texttt{cpratim.25@dartmouth.edu} \\
  \And
  Peter Chin \\
  Department of Engineering\\
  Thayer School of Engineering\\
  \texttt{pc@dartmouth.edu} \\
  \And
  Deepernab Chakrabarty \\
  Department of Computer Science\\
  Dartmouth College\\
  \texttt{deepernab@dartmouth.edu} \\
}

\begin{document}

\maketitle

\begin{abstract}
Understanding which neural components drive specific capabilities in mid-sized language models ($\leq$10B parameters) remains a key challenge. 
We introduce the $(\bm{K}, \epsilon)$-Minimum Sufficient Head Circuit ($K$-MSHC), a methodology to identify minimal sets of attention heads 
crucial for classification tasks as well as Search-K-MSHC, an efficient algorithm for discovering these circuits. Applying our Search-K-MSHC algorithm to Gemma-9B, we analyze three syntactic task families: grammar acceptability,
 arithmetic verification, and arithmetic word problems. Our findings reveal distinct task-specific head circuits, with grammar tasks predominantly utilizing early layers, 
 word problems showing pronounced activity in both shallow and deep regions, and arithmetic verification demonstrating a more distributed pattern across the network. 
 We discover non-linear circuit overlap patterns, where different task pairs share computational components at varying levels of importance. While grammar and arithmetic share many "weak" heads, 
 arithmetic and word problems share more consistently critical "strong" heads. Importantly, we find that each task maintains dedicated "super-heads" with minimal cross-task overlap, 
 suggesting that syntactic and numerical competencies emerge from specialized yet partially reusable head circuits.
\end{abstract}
\section{Introduction}

How do language models organize their capabilities? As mid-sized language models ($\leq$10B parameters) master increasingly diverse tasks—from solving arithmetic problems to identifying ungrammatical sentences—a fundamental question emerges: do related capabilities share the same neural circuitry, or does each task develop its own specialized pathway?

This question has profound implications for how we understand, improve, and control these models. If grammatical analysis and arithmetic reasoning utilize the same circuits, then improvements in one capability might automatically enhance the other. Conversely, if each task relies on distinct components, we gain the ability to target interventions precisely—enhancing specific abilities without disrupting others. Yet despite the practical importance of this organizational question, we lack efficient tools to provide definitive answers~\citep{adolfi2025the}.

To address this, we introduce the \textbf{$(K, \epsilon)$-Minimum Sufficient Head Circuit ($K$-MSHC)} framework, an efficient approach for 
isolating sparse subsets of attention heads that minimally, sufficiently explain task performance. This lets us probe the compositional 
structure of model capabilities in a targeted way. We formalize a novel attention head pruning algorithm which results in approximate circuits that are sufficient and minimal with theoretical guarantees.   

Current interpretability research has revealed that individual attention heads specialize in specific functions \citep{voita2019analyzing,clark2019does} and that disabling certain heads can impact performance on specific tasks \citep{michel2019sixteen}. But these findings don't conclusively answer whether related capabilities emerge from shared or separate neural pathways. The missing piece is a principled approach for identifying the minimal set of model components necessary and sufficient for specific tasks—and importantly, for measuring how these sets overlap across different capabilities.

We build on these results and investigate two further questions: 
\begin{enumerate}[nosep, leftmargin=*]
    \item Do linguistic and numerical reasoning tasks recruit shared pathways or task-specific circuits?
    \item What patterns of overlap emerge between circuits for related tasks, and what do these patterns reveal about the organization of knowledge within the model?
\end{enumerate}

Our experiments reveal a clear organizational pattern in Gemma-9B. Grammar tasks predominantly utilize early layers, word problems engage both shallow and deep network regions, and arithmetic verification employs a distributed mechanism across the network. We observe that while grammatical and arithmetic tasks share many weakly contributing computational components, they maintain dedicated "super-heads" with minimal cross-task overlap. This finding indicates that the model develops specialized circuits for different capabilities while efficiently reusing resources where possible.

These results provide evidence for a nuanced view of LLM organization. Rather than implementing either fully general or fully specialized mechanisms, compact models develop task-specific yet partially reusable circuits. This architectural characteristic enables diverse capabilities within parameter constraints and suggests that similar principles may apply to larger frontier models, where understanding organizational structure becomes increasingly important for alignment and control.

In sum, our contributions are three-fold:
\begin{enumerate}[label=(\roman*), leftmargin=*, itemsep=0pt]
  \item We introduce the \textbf{$(\bm{K}, \epsilon)$-Minimum Sufficient Head Circuit ($\bm{K}$-MSHC)} framework, a highly efficient approach to identify minimal sets of attention heads crucial for specific tasks, measuring their 
  \textit{minimality}, \textit{sufficiency}, and \textit{necessity}.
  \item We develop \textbf{Search-K-MSHC}, an efficient stochastic search algorithm that makes circuit discovery feasible in large language models with thousands of attention heads, complemented by our \textbf{Low-Dimensional Linear Separability (LS)} metric that addresses dimensionality challenges when probing representations.
  \item We present a comprehensive analysis of \textbf{syntactic versus arithmetic circuits} in Gemma-9B, revealing both task-specialized components and shared "super-heads" that 
  demonstrate how mid-sized LLMs efficiently encode multiple capabilities through overlapping head circuits.
\end{enumerate}

\section{Related Work}

\textbf{Evaluating and Probing Language Models.} Specialized benchmarks have emerged to track the capabilities of sub-10B parameter models. 
BLiMP \citep{warstadt2020blimp} provides fine-grained assessments of syntactic competence, 
while GSM8K \citep{cobbe2021training} challenges models with grade-school arithmetic problems. 
Recent surveys \citep{zhao2023survey} document how these capabilities exhibit non-uniform scaling properties across parameter thresholds.
Linear probing methods have become standard tools for analyzing information 
encoded in neural representations \citep{alain2016linear}. Extensions such as control tasks \citep{hewitt2019structural} 
and diagnostic classifiers \citep{tenney2019bert} help distinguish linguistic 
structure from memorization artifacts. Our Low-Dimensional Linear Separability (LS) metric builds on this tradition 
but introduces a crucial innovation by projecting to a minimal subspace via PCA before applying a linear 
classifier.

\textbf{Attention Mechanisms and Mechanistic Interpretability.} Previous work has studied attention pattern interpretability 
\citep{clark2019does} and demonstrated that many heads can be pruned without significant performance degradation 
\citep{michel2019sixteen}. Task-specialized heads have been identified through activation analysis \citep{voita2019analyzing}. 
K-MSHC extends these insights by formalizing 
the concept of minimal sufficient circuits—the smallest set of attention heads where any $K$-subset restores task performance.
The emerging field of mechanistic interpretability seeks to reverse-engineer 
neural networks at the component level. Circuit-level analyses have mapped syntactic processing \citep{clark2019does} 
and key-value memories \citep{geva2021transformer}. Work by Olah et al.\,\citep{olah2020zoom} and Elhage et al.\,\citep{elhage2021superposition} 
suggests that capabilities emerge from sparse subnetworks that can be isolated through careful intervention studies. 
Our K-MSHC framework operationalizes this insight in mid-sized models like Gemma-9B.

\textbf{Computational Component Discovery.} Sparse autoencoders and dictionary--learning techniques have recently been leveraged to extract highly interpretable, 
near-monosemantic features from the activations of large language models \citep{cunningham2023sparse,bricken2023dictionary,gao2024sparse,templeton2024scaling}.
These methods complement circuit-level analyses by working at the representational level and provide an alternative path toward isolating task-relevant computational units.
A parallel line of research proposes algorithms that automatically identify local and global circuits using linear 
computation graphs, cross-layer mappings, or feature editing \citep{marks2024sparse,ge2024linear,lindsey2024crosscoders,dunefsky2025transcoders}. 
Our \emph{Search-K-MSHC} algorithm differs from these approaches by explicitly enforcing $K$-sufficiency, yielding minimal sets that are both necessary and redundantly sufficient for a given task.

\textbf{Task-Specific Head Circuits.} Mechanistic analyses have been extended beyond single-step tasks to multi-hop reasoning in language models 
\citep{yang2024multihop,biran2024hopping,yu2025back} and emergent planning behaviour in specialised agents \citep{jenner2025chess,taufeeque2024sokoban,bush2024planning}. 
Our results on arithmetic word problems echo these findings, showing that high-level reasoning engages heads across distant layers that nonetheless admit mid-sized sufficient subsets.
Recent work traces how language models carry out symbolic or approximate arithmetic, attributing addition to 
Fourier-like or trigonometric transformations and heuristic ensembles \citep{stolfo2023arithmetic,zhou2024fourier,nikankin2024heuristics,kantamneni2025trigonometry}. 
The broadly distributed pattern we observe for arithmetic verification aligns with these results, suggesting that numerical operations recruit a wider basis of heads than purely grammatical processing.
Studies of multilingual models reveal that latent grammatical concepts are encoded in shared subspaces across languages, 
with task-specific specialisations layered on top \citep{brinkmann2025latent,dumas2024llama,zhang2024similar}. The partial head overlap we observe between grammar 
and arithmetic tasks may reflect the same ``semantic hub'' principle observed in these cross-lingual analyses.

\section{Methodology}
\label{sec:methods}

We formalize the problem of identifying minimal head circuits responsible for specific model capabilities. Our $(K, \epsilon)$-Minimum Sufficient Head Circuit framework quantifies the causal contribution of attention heads to task performance.

\subsection{The \texorpdfstring{$(K, \epsilon)$}{K-ε}-Minimum Sufficient Head Circuit}

To identify where task-specific knowledge resides in a model, we adopt a causal intervention approach: if removing specific components disrupts the ability to distinguish correct from incorrect examples, those components likely encode the relevant knowledge.

\begin{definition}[Task Separability Score]
For a task $\mathcal{T}$ with dataset $\mathcal{D}_\mathcal{T} = \{(x_i, y_i)\}_{i=1}^n$ where $y_i \in \{-1, 1\}$, the separability score of a model configuration $\mathcal{M}$ is:

\begin{equation}
  \mathcal{S}_{\mathcal{D}_\mathcal{T}}(\mathcal{M}) = \frac{1}{n} \sum_{i=1}^n \mathbb{I}\left[f_\theta(\mathbf{h}_{x_i,L}^{\text{EOS}}) = y_i\right]
\end{equation}

where $\mathcal{M}$ specifies which attention heads are active, $\mathbf{h}_{x_i,L}^{\text{EOS}}$ is the final-layer embedding representation, and $f_\theta: \mathbb{R}^d \rightarrow [0, 1]$ is an optimal classifier over that embedding space.
\end{definition}
We hypothesize that each task relies on a minimal circuit of attention heads that maintains these distinctions. By comparing separability across head subsets, we can identify these circuits and analyze how language models allocate resources between capabilities.

\begin{definition}[(K, $\epsilon$)-Minimum Sufficient Head Circuit]
The $(K, \epsilon)$-Minimum Sufficient Head Circuit identifies the smallest set of attention heads $\mathcal{H} \subset \mathcal{M}$ such that any subset of $K$ heads from $\mathcal{H}$ 
can restore the model's classification performance to over some $\epsilon$ performance threshold.
\end{definition}

Given a baseline model $\mathcal{B}$ (typically a subset of heads in $\mathcal{M}$) and a parameter $\epsilon \in [0, 1]$, we define the understanding threshold:
\begin{equation}
  \label{eq:threshold}
  \mathcal{U}^\epsilon(\mathcal{M}, \mathcal{B}) = \mathcal{S}_{\mathcal{D}_\mathcal{T}}(\mathcal{B}) + \epsilon \cdot (\mathcal{S}_{\mathcal{D}_\mathcal{T}}(\mathcal{M}) - \mathcal{S}_{\mathcal{D}_\mathcal{T}}(\mathcal{B}))
\end{equation}

\begin{definition}[$(K, \epsilon)$-Minimum Sufficient Head Circuit]
The $(K, \epsilon)$-Minimum Sufficient Head Circuit is the smallest set of heads $\mathcal{H} \subset \mathcal{M}$ satisfying:
\begin{equation}
  \label{eq:mshc-definition}
  \forall \mathcal{H}' \subseteq \mathcal{H} \text{ with } |\mathcal{H}'| = K: \mathcal{S}_{\mathcal{D}_\mathcal{T}}((\mathcal{M} \setminus \mathcal{H}) \cup \mathcal{H}') \geq \mathcal{U}^\epsilon(\mathcal{M}, \mathcal{B})
\end{equation}
\end{definition}

This definition has three key properties:
\begin{property}[Key Properties of $(K, \epsilon)$-MSHC] \
\begin{enumerate}[label=(\roman*), leftmargin=*, itemsep=0pt]
  \item \textbf{Minimality}: $\mathcal{H}$ is the smallest set satisfying the condition
  \item \textbf{$K$-Sufficiency}: Any $K$-subset of $\mathcal{H}$ can restore performance
  \item \textbf{Isolated Insignificance}: $\mathcal{M} \setminus \mathcal{H}$ alone contributes minimally to task understanding
\end{enumerate}
\end{property}

The parameter $K$ controls redundancy in the circuit, while $\epsilon$ determines how close to full performance the circuit must restore.

\subsection{Search-K-MSHC: An Efficient Algorithm for Circuit Discovery}

Finding the exact $(K, \epsilon)$-Minimum Sufficient Head Circuit is likely computationally intractable, requiring examination 
on the order of $2^{|\mathcal{M}|}$ possible head subsets. We introduce Search-K-MSHC, a stochastic algorithm with parameters $\mathcal{W}$ 
(window size), $p$ (percentile), and $N$ (samples) that efficiently approximates the solution through two phases:
\begin{enumerate}[label=(\roman*), leftmargin=*, itemsep=0pt]
  \item \textbf{Macro Layer Search}: Identifies critical layers by window-based ablation.
  \item \textbf{Micro Head Search}: Refines head candidates via binary search and stochastic pruning.
\end{enumerate}
\textbf{Macro Layer Search.} We identify task-critical layers via:
\begin{enumerate}[label=(\roman*), leftmargin=*, itemsep=0pt]
  \item Window-based ablation: Slide a window of size $\mathcal{W}$ across adjacent network layers
  \item Initialize array $\textsc{Drop}[1\!:\!L]$ to store performance drop measurement. Update $\textsc{Drop}[\ell] = \min\left(\textsc{Drop}[\ell], \mathcal{S}_{\mathcal{D}_\mathcal{T}}(\mathcal{M}) - \mathcal{S}_{\mathcal{D}_\mathcal{T}}(\mathcal{M} \setminus \mathcal{H}_{\mathcal{W}})\right)$ where $\mathcal{H}_{\mathcal{W}}$ contains all heads in the window
  \item Select high-impact layers (top $p$-th percentile) for candidate set $\mathcal{C}$ from $\textsc{Drop}[1\!:\!L]$
\end{enumerate}
\textbf{Stochastic Head Pruning.} With baseline $\mathcal{B} = \mathcal{M} \setminus \mathcal{C}$, we:
\begin{enumerate}[label=(\roman*), leftmargin=*, itemsep=0pt]
  \item Start with subset size $k = \lfloor|\mathcal{C}|/2\rfloor$
  \item Sample $N$ random $k$-subsets and find worst-performing $\Theta_{\min} = \arg\min_{i} \mathcal{S}_{\mathcal{D}_\mathcal{T}}(\mathcal{B} \cup \Theta_i)$
  \item If $\mathcal{S}_{\mathcal{D}_\mathcal{T}}(\mathcal{B} \cup \Theta_{\min}) < \mathcal{U}^{\epsilon}(\mathcal{M}, \mathcal{B})$, prune $\mathcal{C} \gets \mathcal{C} \setminus \Theta_{\min}$; else reduce $k \gets \max(K, \lfloor k/2 \rfloor)$
  \item Terminate when $k=K$ and threshold met
\end{enumerate}

\subsection{Theoretical Analysis of Search-K-MSHC}

\begin{definition}
Let $\mathcal{C}_{i}$ be the candidate set at iteration $i$, and let $\tau = \mathcal{U}^{\epsilon}(\mathcal{M}, \mathcal{B})$ be our performance threshold. For parameters $0 \leq \delta_i, \delta_T \leq 1$, we define:
\begin{enumerate}[label=(\roman*), leftmargin=*, itemsep=0pt]
  \item \textit{Low-impact heads}: Set $\mathcal{Q} \subseteq \mathcal{C}_{i}$ with $|\mathcal{Q}| = \delta_i |\mathcal{C}_{i}|$ s.t. $\forall \mathcal{X} \subseteq \mathcal{C}_i$ with $|\mathcal{X}| = K$ and $|\mathcal{X} \cap \mathcal{Q}| > \delta_T K$: $\mathcal{S}_{\mathcal{D}_\mathcal{T}}((\mathcal{M} \setminus \mathcal{C}_{i}) \cup \mathcal{X}) \leq \tau$,
  or heads that don't meaningfully contribute to "understanding"
  
  \item \textit{Prunable sets}: $\mathcal{P}_i = \{\mathcal{X} \subseteq \mathcal{C}_i : |\mathcal{X}| = K, |\mathcal{X} \cap \mathcal{Q}| > \delta_T K\}$, the $K$-subsets containing too many low-impact heads
  that we would like to prune.
\end{enumerate}
\end{definition}
These assumptions are made to simplify the analysis but likely emulate actual head importance patterns, with a number of unimportant heads 
over some threshold likely leading to prunable sets.

\begin{theorem}[Expected Missed Prunable Sets]
\label{thm:missed_sets}
For a candidate set $\mathcal{C}_i$ with a contamination rate $\delta_i$ of low-impact heads and threshold parameter $\delta_T < \delta_i$, the expected number of prunable sets that remain undetected after sampling $N$ random $K$-subsets is:
\begin{equation}
\mathbb{E}[\text{Undetected Sets}] = O\left(\frac{\delta_i |\mathcal{C}_i|}{\delta_T K} \cdot \exp\left(-NK (\delta_i - \delta_T)^2\right)\right)
\end{equation}
\end{theorem}

\begin{proof}
For $K$ randomly sampled heads, the probability of missing a prunable set follows a hypergeometric distribution $H \sim \text{Hypergeometric}(|\mathcal{C}|, |\mathcal{Q}|, K)$. By Hoeffding's inequality:
\begin{equation}
  \Pr[H \leq \delta_T \cdot K] \leq \exp\left(-2K \cdot (\delta_i - \delta_T)^2\right)
\end{equation}

With $N$ independent samples, the miss probability becomes:
\begin{equation}
  \Pr[\text{All $N$ Samples Miss}] \leq \exp\left(-2NK \cdot (\delta_i - \delta_T)^2\right)
\end{equation}

Since we can only catch non-overlapping prunable sets, the maximal number of undetected prunable sets is bounded by $|\mathcal{Q}| / (\delta_T \cdot K + 1)$. By linearity of expectation:
\begin{align}
  \mathbb{E}[\text{Undetected Sets}] &\leq \frac{|\mathcal{Q}|}{\delta_T \cdot K + 1} \cdot \exp\left(-2NK \cdot (\delta_i - \delta_T)^2\right) \\
  &= O\left(\frac{\delta_i |\mathcal{C}_i|}{\delta_T K} \cdot \exp\left(-NK (\delta_i - \delta_T)^2\right)\right) \tag{since $|\mathcal{Q}| = \delta_i |\mathcal{C}_i|$}
\end{align}
\end{proof}
\vspace{-1cm}
The key takeaways from this result are two-fold:
\begin{enumerate}[label=(\roman*), leftmargin=*, itemsep=0pt]
  \item The expected number of missed prunable sets decreases exponentially with $N$ (samples), $K$ (subset size), and $\Delta_i^2 = (\delta_i - \delta_T)^2$ (squared margin).
  \item Critically, when the low-impact ratio $\delta_i$ is much higher than the threshold $\delta_T$, the algorithm is extremely unlikely to terminate with significant numbers of prunable sets remaining.
\end{enumerate}

\begin{algorithm}[H]
    \setstretch{1.3}
    \caption{Search-K-MSHC}
    \label{alg:mshc}
    \begin{algorithmic}[1]
    \Require window size $W$, percentile $p$, samples per iteration $N$
    \State initialise array $\textsc{Drop}[1\!:\!L]\gets 0$
    \For{$s = 1$ \textbf{to} $L-W+1$} \Comment{\textbf{Macro layer search}}
        \State $\mathcal{H}_W\gets$ heads in layers $s{:}s+W-1$
        \For{$\ell = s$ \textbf{to} $s+W-1$}
            \State $\textsc{Drop}[\ell]=\min\!\bigl(\textsc{Drop}[\ell],\mathcal{S}_{\mathcal{D}_\mathcal{T}}(\mathcal{M}) - \mathcal{S}_{\mathcal{D}_\mathcal{T}}(\mathcal{M} \setminus \mathcal{H}_W)\bigr)$
        \EndFor
    \EndFor
    \State $\mathcal{C}\gets\{\mathcal{H}_\ell\mid\textsc{Drop}[\ell]\ge\text{top }p\text{-th percentile of }\textsc{Drop}\}$ \Comment{sensitive layers}
    \State $\mathcal{B}\gets \mathcal{M} \setminus \mathcal{C}$ \Comment{baseline model}
    \State $k\gets\left\lfloor |\mathcal{C}| / 2 \right\rfloor$
    \While{$k \ge K$} \Comment{\textbf{Stochastic head pruning}}
        \State $\mathrm{best}\gets 1$;\; $\Theta_{\mathrm{min}}\gets\emptyset$
        \For{$i = 1$ \textbf{to} $N$}
            \State draw $\Theta\sim\mathrm{Unif}\{\mathcal{X}\subseteq\mathcal{C}:|\mathcal{X}|=k\}$
            \If{$\mathcal{S}_{\mathcal{D}_\mathcal{T}}(\mathcal{B} \cup \Theta) < \mathrm{best}$}
                \State $\mathrm{best}\gets\mathcal{S}_{\mathcal{D}_\mathcal{T}}(\mathcal{B} \cup \Theta)$;\; $\Theta_{\mathrm{min}}\gets\Theta$
            \EndIf
        \EndFor 
        \If{$\mathrm{best} \le \mathcal{U}^{\mathcal{S}}_{\mathcal{D}_\mathcal{T}}(\mathcal{M}, \mathcal{B})$}
            \State $\mathcal{C}\gets\mathcal{C} \setminus \Theta_{\mathrm{min}}$ \Comment{safe to prune}
        \Else
            \State $k\gets\max(K,\lfloor k/2\rfloor)$
        \EndIf
    \EndWhile
    \State \Return $\mathcal{C}$
    \end{algorithmic}
    \end{algorithm}

\subsection{Scoring Metric: Low-Dimensional Linear Separability (LS)}
\label{subsec:lsmetric}

Standard linear probes for analyzing LLM representations often overfit due to high dimensionality ($d \approx 4096$) and limited training data. 
To address this challenge, we introduce Low-Dimensional Linear Separability (LS), focused on the final layer's EOS token representations
as our scoring metric to guide the search for minimal sufficient head circuits. It operates in two phases:

\textbf{Dimensionality Reduction.} We project to a subspace preserving maximal variance by computing the empirical covariance matrix $\Sigma_L = \frac{1}{|\mathcal{D}_\mathcal{T}|}\sum_{x}\left(\mathbf{h}_{x,L}^{\text{EOS}} - \bar{\mathbf{h}}_{L}\right)\left(\mathbf{h}_{x,L}^{\text{EOS}} - \bar{\mathbf{h}}_{L}\right)^{\top}$, extracting its top $D$ eigenvectors $\mathbf{W}_L$, and projecting:
\begin{equation}
  \tilde{\mathbf{h}}_{x,L} = \mathbf{W}_L^\top(\mathbf{h}_{x,L}^{\text{EOS}} - \bar{\mathbf{h}}_{L}) \in \mathbb{R}^D
\end{equation}

\textbf{Classification.} We train a linear SVM by optimizing the regularized hinge loss:
\begin{equation}
  \min_{\mathbf{w},b}\; \frac{1}{2}\|\mathbf{w}\|_2^2 +\;C\sum_{i=1}^{|\mathcal{D}_\mathcal{T}|} \max\left(0,\,1-y_i(\mathbf{w}^\top\tilde{\mathbf{h}}_{x_i,L}+b)\right), \quad C=10
\end{equation}
and calculate LS score as classification accuracy: $\operatorname{LS}^{D}_{\mathcal{D}_\mathcal{T}} = \frac{1}{|\mathcal{D}_\mathcal{T}|} \sum_{i=1}^{|\mathcal{D}_\mathcal{T}|}\mathbb{I}[\operatorname{sign}(\mathbf{w}^\top\tilde{\mathbf{h}}_{x_i,L}+b)=y_i]$

By restricting to $D \leq 5$ dimensions, we ensure the metric captures task-relevant information rather than dimensionality artifacts. This approach efficiently detects when attention heads encode task-specific knowledge in their representations.
\subsection{Evaluation Framework: Task Families and Dataset Construction}
\label{subsec:tasks}

We evaluate K-MSHC using three task families with controlled minimal pairs, where examples differ only in a single task-relevant feature. All tasks have balanced classes and use a consistent yes/no formulation.

\textbf{Grammar Acceptability (G).} Based on BLiMP \citep{warstadt2020blimp} 
(67,000 sentence pairs), focusing on determiner-noun agreement
due to the inherent numerocity of the task:

\begin{examplebox}\footnotesize
\textbf{Correct:} Leslie isn't firing that actress.\\[2pt]
\textbf{Incorrect:} Leslie isn't firing that actresses.
\end{examplebox}

\textbf{Arithmetic Verification (A).} 1000 algorithmically generated equation pairs with random operands $n_1, n_2 \in [1, 10^3]$, operations $\in \{+, -\}$, and perturbed answers (0.5 / 1.5$\times$) for incorrect equations:

\begin{examplebox}\footnotesize
\textbf{Correct:} 1338 + 88 = 1426\\[2pt]
\textbf{Incorrect:} 1338 + 88 = 2139 \quad \textit{($\approx$ 1.5 $\times$ correct result)}
\end{examplebox}

\textbf{Word Problems (W).} 100 natural language arithmetic problems templates filled with random numbers using the same methodology as (A), with (0.5 / 1.5$\times$) perturbed answers for incorrect equations:

\begin{examplebox}\footnotesize
\textbf{Correct:} Tim has 5 apples and eats 2, leaving him with 3 apples.\\[2pt]
\textbf{Incorrect:} Tim has 5 apples and eats 2, leaving him with 5 apples. \quad \textit{($\approx$ 1.5 $\times$ correct result)}
\end{examplebox}

This progression from grammatical knowledge (G) to abstract computation (A) to contextualized reasoning (W) allows analysis of both task-specific circuits and potential shared components.

\section{Experiments}

Using our K-MSHC framework, we probe Gemma-9B to analyze which head circuits are responsible for the models understanding of grammar, arithmetic, and word problems tasks.
We run experiments with parameters $\mathcal{W}=5$, $p=0.75$, $N=10$, $K=10$, and $\epsilon=0.25$ 
across 20 trials with mini-batches of 50 (positive and negative) examples per task.
All experiments were conducted on Nvidia H100 GPUs with 50 GB of memory.
We used PyTorch for all LLM inference and manipulation, while scikit-learn was employed for implementing PCA dimensionality reduction and SVM classifiers for the Linear Separability metrics.
\subsection{Baseline Performance Analysis}
\label{subsec:protocol}
Before identifying specific head circuits, we first establish that information about our classification tasks is indeed encoded in the model's representations. Table~\ref{tab:baseline_performance} presents the Linear Separability (LS) scores for each task, showing that Gemma-9B's representations naturally encode strong task-relevant information, with baseline LS scores ranging from 0.77 to 0.99. When critical layers (identified through ablation) are removed, performance drops substantially across all tasks, with different tasks showing varying sensitivity to layer ablation. Arithmetic verification exhibits the largest drop (44\%), followed by grammar (36\%) and word problems (21\%).
\begin{table}[h]
\centering
\begin{tabular}{lccc}
\toprule
\textbf{Task} & \textbf{Baseline LS Score} & \textbf{LS Score Post 25\% Layer Ablation} & \textbf{Drop} \\
\midrule
Arithmetic & $0.99\ [0.98, 1.00]$ & $0.55\ [0.52, 0.60]$ & 44\% \\
Grammar & $0.86\ [0.78, 0.90]$ & $0.50\ [0.49, 0.52]$ & 36\% \\
Word Problems & $0.77\ [0.73, 0.86]$ & $0.56\ [0.53, 0.61]$ & 21\% \\
\bottomrule
\end{tabular}
\vspace{0.5cm}
\caption{Linear separability scores before and after ablating critical layers, 
showing task-dependent information encoding within the model architecture. 
Values show medians with 95\% confidence intervals in square brackets.}
\label{tab:baseline_performance}
\end{table}
\vspace{-0.5cm}
\subsection{Distinct Head Circuits for Different Capabilities}
\begin{figure}[htbp]
    \centering
    \includegraphics[width=\textwidth]{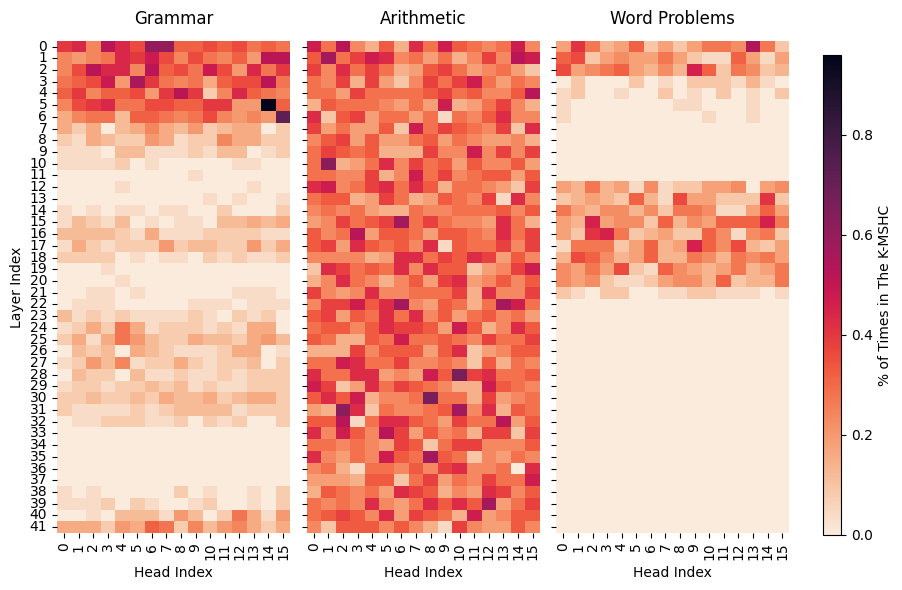}
    \vspace{-0.2cm}
    \caption{Heat map of attention head importance across model layers. Each cell represents an attention head, with color intensity showing selection frequency across 20 trials.}
    \label{fig:head_distribution}
\end{figure}
Figure \ref{fig:head_distribution} reveals the spatial distribution of attention heads identified by our K-MSHC algorithm across the model architecture. 
Analysis of these patterns uncovers three key organizational principles in Gemma-9B's computational structure.

\textbf{Architectural Specialization:} Each task engages a distinctive subset of the network. Grammar processing primarily activates early layers (0-6) with 
specific banding patterns in later regions. Word problems utilize a bimodal distribution with concentrated activity in both shallow (0-3) and deep (11-20) regions. 
Arithmetic verification shows the most distributed pattern, engaging heads across the entire network with more uniform density.
    
\textbf{Intra-Layer Selectivity:} Within even the most important layers, not all heads contribute equally. Attention patterns show high selectivity, 
with only a few heads per layer being consistently critical. This suggests a sparse coding principle where specific head combinations—rather
 than entire layers—form the building blocks of task-specific circuits.
    
\textbf{Task-Dependent Organization:} Structurally related tasks demonstrate different patterns of head criticality. The well-defined task 
boundaries of grammar and arithmetic verification correlate with concentrated "super-head" patterns—specific heads that appear in almost all 
K-MSHCs for these tasks. In contrast, word problems require contextual reasoning across both linguistic and numerical domains, resulting 
in more diffuse activation patterns without dominant super-heads.

These findings demonstrate that language capabilities are not uniformly distributed throughout the model but are encoded through sparse, task-specialized circuits with distinct architectural signatures.

\subsection{Circuit Overlap Analysis}
To investigate how computational resources are shared across tasks, we analyzed the overlap between circuits identified for each task pair. Figure \ref{fig:head_overlap} visualizes this overlap at different selection thresholds (percentage of most frequently selected heads):
\begin{figure}[htbp]
    \centering
    \includegraphics[width=\textwidth]{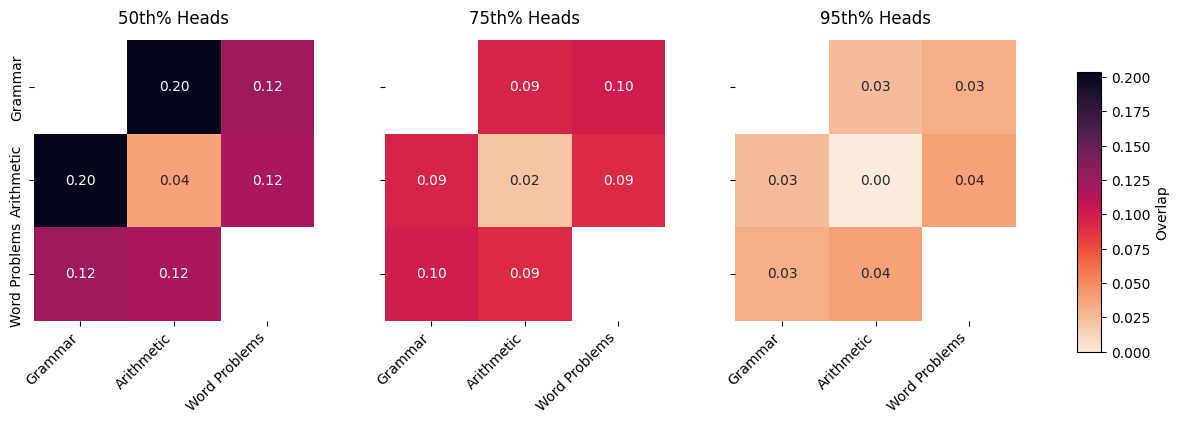}
    \vspace{-0.2cm}
    \caption{Circuit overlap (Jaccard similarity) between task pairs. The matrix shows overlap percentages at three selection thresholds: 50\% (weaker but more numerous heads), 75\% (moderate), and 95\% (strongest "super-heads"). The central region indicates three-way overlap.}
    \label{fig:head_overlap}
\end{figure}

Our analysis reveals several nuanced patterns in how Gemma-9B shares computational resources across tasks.

\textbf{Sharing "Weak" vs "Strong" Heads:} Task pairs exhibit different sharing patterns depending on head importance. 
    At the 50\% threshold (including weaker heads), grammar and arithmetic show substantial overlap ($\approx$20\%). However, this relationship inverts at the 75\% threshold, 
    where arithmetic and word problems share more critical heads ($\approx$10\%) than grammar-arithmetic pairs ($\approx$9\%). This finding suggests that tasks 
    either share "weak" heads that are vaguely related or share "strong" heads that are more specific to the task overlap.
    
\textbf{Specialized "Super-Heads":} The strongest heads (95\% threshold) show minimal cross-task overlap, with each task pair maintaining dedicated circuits of "super-heads".
     This separation is particularly striking given that arithmetic verification and word problems involve related numerical reasoning, yet their most critical components remain largely distinct. 
     This finding challenges the notion that higher-level capabilities like "arithmetic reasoning" have a single circuit implementation within the model.
    
\textbf{Non-Uniform Resource Allocation:} The asymmetric pattern of overlap reduction across thresholds—with grammar-arithmetic showing steeper decline than other pairs—indicates that head importance has task-specific scaling properties. This suggests that the model allocates computational resources non-uniformly, with some task relationships maintaining more robust sharing across importance levels than others.
These overlap patterns reveal that Gemma-9B balances specialization and resource sharing through a hierarchical organization of computational components.
 While allowing partial resource reuse, particularly for related tasks, the model also maintains dedicated circuits for each capability's core processing requirements.

\section{Limitations and Future Work}
While our K-MSHC framework reveals important insights about attention circuit organization, several limitations should be acknowledged:

\textbf{Algorithmic Assumptions.} Our theoretical analysis relies on simplifying assumptions about head contamination rates and their distribution across the network. The convergence guarantees are strongest when the distribution of low-impact heads is well-separated from task-critical heads, but real-world models may exhibit more complex patterns of head importance with less clear separation.

\textbf{Circuit Specificity.} The triangulation of head circuits was not as precise as initially anticipated. We observed some variation in the specific heads identified across trials, suggesting that multiple distinct but functionally equivalent circuits may exist for each task. This redundancy makes it challenging to definitively map the exact set of heads responsible for a capability.

\textbf{Model and Parameter Sensitivity.} Our analysis is limited to a single model architecture (Gemma-9B) with one set of hyperparameters for the search algorithm. While we found our approach effective, the circuits identified may be sensitive to the choice of $K$, $\epsilon$, and other parameters, particularly for tasks where performance is distributed across many weakly-contributing heads.

\textbf{Limited Task Diversity.} Our focus on three specific task families provides an insightful but still limited view of how language capabilities are organized. More complex reasoning, world knowledge, or multimodal tasks might reveal different circuit structures and overlap patterns.
Future work should address these limitations by:
\begin{enumerate}[label=(\roman*), leftmargin=*, itemsep=0pt]
  \item Exploring broader parameter settings across different model architectures to identify invariant circuits and understand sensitivity effects
  \item Developing refined methods for handling redundant circuits while investigating their activation dynamics during inference
  \item Extending analysis to more complex tasks that engage multiple capabilities simultaneously to better map the functional organization of language models
\end{enumerate}

These extensions would strengthen our understanding of how language capabilities are mechanistically implemented in neural architectures and potentially enable more targeted model interventions and improvements.

\section{Conclusion}

Our work introduces K-MSHC, a framework for identifying minimal sufficient head circuits in mid-sized language models, revealing that different syntactic tasks utilize distinct neural pathways in Gemma-9B. Grammar tasks predominantly activate early layers (0-6), while word problems utilize specific bands in both early and deep regions, with arithmetic verification showing more distributed patterns. We found non-linear patterns of circuit overlap, with grammar and arithmetic sharing more "weak" heads while arithmetic and word problems share more "strong" heads, indicating that despite partial resource sharing, each task maintains dedicated "super-heads" with minimal overlap at high thresholds. These findings advance mechanistic interpretability by demonstrating that language capabilities emerge from specialized but partially reusable circuits rather than fully general mechanisms, suggesting that future research should focus on identifying and understanding these sparse computational primitives across different model scales and architectures.

\bibliographystyle{plainnat}
\bibliography{references}

\end{document}